\documentclass[journal]{IEEEtran}

\ifCLASSINFOpdf
\else
\fi

\usepackage{amsthm}
\usepackage{amsmath,amssymb}
\usepackage{amsfonts}
\usepackage{graphicx,subfigure}
\usepackage{float}
\usepackage{relsize}
\usepackage{mathtools}
\usepackage{graphics}
\usepackage{booktabs}
\usepackage{multirow}
\usepackage{cite}
\usepackage{caption}
\usepackage[usenames, dvipsnames]{color}
\usepackage{graphicx} 
\usepackage{bm}
\newcommand{\EQ}{\begin{eqnarray}}
\newcommand{\EN}{\end{eqnarray}}
\newtheorem{assumption}{Assumption}

\newcommand{\floor}[1]{\lfloor #1 \rfloor}

\usepackage{tabulary}
\newtheorem{thm}{Theorem}

\newtheorem{corollary}{Corollary}
\usepackage[algo2e]{algorithm2e}
\usepackage{algorithm} 
\usepackage{algorithmic}  
\begin{document}
\title{Accumulated Decoupled Learning: Mitigating Gradient Staleness in Inter-Layer Model Parallelization}

\author{
	\IEEEauthorblockN{Huiping Zhuang, 
		Zhiping Lin, Kar-Ann Toh
	}
}
\maketitle

\begin{abstract}
Decoupled learning is a branch of model parallelism which parallelizes the training of a network by splitting it depth-wise into multiple modules. Techniques from decoupled learning usually lead to stale gradient effect because of their asynchronous implementation, thereby causing performance degradation. In this paper, we propose an accumulated decoupled learning (ADL) which incorporates the gradient accumulation technique to mitigate the stale gradient effect. We give both theoretical and empirical evidences regarding how the gradient staleness can be reduced. We prove that the proposed method can converge to critical points, i.e., the gradients converge to 0, in spite of its asynchronous nature. Empirical validation is provided by training deep convolutional neural networks to perform classification tasks on CIFAR-10 and ImageNet datasets. The ADL is shown to outperform several state-of-the-arts in the classification tasks, and is the fastest among the compared methods.
\end{abstract}

\section{Introduction}
Deep neural networks (DNN), including convolutional neural network (CNN) \cite{lecun1998gradient} and recurrent neural network (RNN) \cite{hochreiter1997long}, have gained significant achievements in a variety of complex tasks. Unlike traditional machine learning techniques, DNNs tend to perform increasingly well given deeper and wider structures \cite{zagoruyko2016wide,huang2017densely,gastaldi2017shake}. However, such performance improvement can be costly as it needs a much longer training time. \textit{Data parallelism} \cite{sergeev2018horovod} and \textit{model parallelism} \cite{jia2018beyond} are two common solutions to reducing the consumed time through parallelizing the computation during network training. 

Data parallelism employs multiple workers with each worker handling a replica of the entire network for processing a subset of the training data. This type of parallelization has been well explored, which thrives in networks with high computation-communication ratio such as ResNet-like structures \cite{he2016deep,hu2018squeeze,xie2017aggregated}. Model parallelism, on the other hand, splits the network into several parts with each part handled by one specific worker. Such parallelism spawns various forms. For instance, model parallelism can be adopted in CNNs by parallelizing the convolution operations \cite{yadan2013multi,alex2012imagenet}. Another relatively new form of model parallelism is \textit{decoupled learning} \cite{jaderberg2017decoupled}. This technique partitions a network in a much simpler way by splitting it depth-wise into multiple modules---with each module containing a stack of layers---to facilitate inter-layer module-wise parallelization. Unlike other model parallelism counterparts that usually demand an extensive alteration for realization, the simple depth-wise partition of the decoupled learning encourages a straightforward implementation on various platforms with minimum effort. Such property is worth further exploration.

The decoupled learning is achieved through bypassing the need for a global backpropagation (BP) which has been a standard practice for training networks due to DNN's highly non-convex nature. Specifically, we have to address the lockings \cite{jaderberg2017decoupled} (i.e., the forward, backward, and update lockings) inherited from the BP procedure. These lockings prohibit the network modules from behaving  asynchronously, and lead to inefficiency as the majority of a network is kept idle during training. There have been various attempts to achieve decoupled learning by removing one or more of these lockings. These attempts can be categorized into two groups: the local error learning (LEL) based methods, and delayed gradient (DG) based methods. 

The LEL-based methods build auxiliary networks to generate local error gradients. They sever the gradient flow between the adjacent modules, thereby avoiding the global BP. The difficulty behind methods in this group lies in the design of the auxiliary networks, which appear to be network-specific as well as task-specific \cite{belilovsky2018greedy}. In general, the LEL-based methods give worse performance compared with their BP counterparts without heavy-weight auxiliary networks that would need much longer training time.

The DG-based methods attain decoupled learning by updating the network modules with delayed gradients (or known as ``older'' gradients). These methods begin at unlocking the backward pass \cite{huo2018decoupled} of BP, and are advanced to be lock-free \cite{zhuang2019fully}. Decoupling the learning with DGs is more propitious than the LEL-based methods as it usually gives comparable performance with the BP baselines. The current development of the DG-based methods is constrained by its split size (i.e., the number of modules a network can be split into). To the best of our knowledge, the maximum split size reported in the literature is only 4 modules. Such a limited capacity  is mainly caused by the \textit{stale gradient effect} \cite{zheng2017asynchronous} (also known as \textit{gradient staleness}) that becomes more serious with larger split size. This effect could lead performance drop \cite{zheng2017asynchronous} or even divergence \cite{xu2020acceleration}. To increase the split capacity, the key is to reduce the gradient staleness. In this paper, we propose an accumulated decoupled learning (ADL), which incorporates the gradient accumulation (GA) technique in the split modules to mitigate the stale gradient effect. The contributions of this work include:
\begin{itemize}
	\item Proposal of a new model parallelism technique, achieved by addressing the locking problems in BP.
	\item Incorporation of GA technique into the decoupled learning, which is shown theoretically and empirically to reduce the delayed gradient effect.
	\item Convergence analysis showing that our method can converge to critical points, i.e., the gradients have a lower bound that converges to $0$.
	\item Experiments that include CIFAR-10 and ImageNet classification tasks. We show that the proposed method gives comparable or better classification as well as acceleration performance. In particular, the proposed method can train networks with a split size up to 10, which is significantly larger than the maximum 4 in the previous arts.
\end{itemize}

\section{Related Works}

\subsection{Local Error Learning Based Methods}
The key feature of LEL-based methods is the design of auxiliary networks. The decoupled neural interface (DNI) \cite{jaderberg2017decoupled} adopts a local network that generates synthetic error gradients to achieve decoupled learning. The DNI gives a lock-free training of DNNs but its performance has been shown to degrade quickly or even diverge in training deeper networks \cite{huo2018decoupled}. A local classifier \cite{mostafa2018deep} is adopted to generate local gradients, but it performs constantly worse than a standard BP. A method called pred-sim \cite{nokland2019training} incorporating a cross-entropy loss and a similarity measure successfully trains several VGG networks with comparable performance to the BP baselines. However, the pred-sim method has not been verified in deeper networks. The decoupled greedy learning (DGL) \cite{belilovsky2019decoupled} achieves the decouple learning through designing a light-weight auxiliary network. In general, these LEL-based methods involve a sophisticated auxiliary design, which adds further burden to the tediousness of hyperparameter tuning.

\subsection{Delayed Gradient Based Methods}
The DG-based methods attain decoupled learning by updating the network modules with delayed gradients. A decoupled parallel BP with delayed gradients (DDG) \cite{huo2018decoupled} addresses the backward locking, and shows comparable classification performance to the BP baselines on several ResNet structures. Since the DDG only unlocks the backward pass, the acceleration gained by model parallelism is relatively limited (e.g., $\approx$2$\times$ speedup with 4 GPUs). In \cite{huo2018training}, another backward-unlocking technique, the feature replay (FR), is introduced, which slightly outperforms the DDG. The fully decoupled method with delayed gradients (FDG) \cite{zhuang2019fully} further addresses the forward and the update lockings, achieving a lock-free decoupled learning. Recently, a technique called DSP \cite{xu2020acceleration} has also attained a lock-free decoupled learning. However, these prior arts using DGs inevitably suffer from the stale gradient effect, which becomes more apparent as the spit size grows.

\subsection{Asynchronous Stochastic Gradient Descent}
The asynchronous stochastic gradient descent (ASGD) based methods \cite{dean2012large,lian2015asynchronous} also adopt DGs to facilitate asynchronous distributed learning. They belong to the area of data parallelism since each worker handles the calculation of gradients based on the whole network. By involving DGs, likewise the ASGD-based methods suffer from the stale gradient effect. In \cite{zheng2017asynchronous}, a gradient compensation is made to deal with this effect. This leads to certain improvement compared with the traditional methods, though the stale gradient effect is still quite prominent when the number of workers is large.
\section{Preliminaries}
Here, we revisit some background knowledge for training a feedforward neural network, including the GA technique adopted in our proposed method. During this revisit, the BP lockings \cite{jaderberg2017decoupled} as well as the stale gradient effect are also explained. 

\subsection{Backpropagation and Lockings}
Assume that we need to train an $\mathcal{L}$-layer network. The $l^{\text{th}}$ ($1\le l\le \mathcal{L}$) layer produces an activation $\bm{z}_{l} = F_{l}(\bm{z}_{l-1}; \bm{\theta}_l)$ by taking $\bm{z}_{l-1}$ as its input, where $F_{l}$ is an activation function and  $\bm{\theta}_{l}\in \mathbb{R}^{n_l}$ is weight vector in layer $l$. The sequential generation of the activations results in the \textit{forward locking}  since $\bm{z}_{l}$ depends on its previous layers. Let $\bm{\theta} = [\bm{\theta}_{1}^{T}, \bm{\theta}_{2}^{T}, ..., \bm{\theta}_{\mathcal{L}}^{T}]^{T} \in \mathbb{R}^{\Sigma_{i=1}^{\mathcal{L}}n_i}$ denote the parameter vector of the entire network. Assume $f$ is a loss function. Training the feedforward network can then be formulated as
\begin{align}\label{eq_optimization}
\underset{\bm{\theta}}{\text{minimize}} \quad f_{\bm{x}}(\bm{\theta})
\end{align}
where $\bm{x}$ represents the entire input-label information (or the entire dataset). In the rest of this paper, we shall use $f(\bm{\theta})$ to represent $f_{\bm{x}}(\bm{\theta})$ for convenience.

The gradient descent algorithm is often used to solve \eqref{eq_optimization} by updating the parameter $\bm{\theta}$ iteratively as follows:
\begin{align}\label{eq_gd_batch}
\bm{\theta}^{t+1} = \bm{\theta}^{t} - \gamma_{t}\bm{\bar g}_{\theta}^{t}
\end{align}
or equivalently,
\begin{align}\label{eq_gd_update}
\bm{\theta}_{l}^{t+1} = \bm{\theta}_{l}^{t} - \gamma_{t}\bm{\bar g}_{\bm{\theta}_l}^{t}, \ l=1,...,\mathcal{L}
\end{align}
where $\gamma_{t}$ is the learning rate. Index $t$ here usually implies the \textit{batch index}, with $\bm{\bar g}_{\bm{\theta}_l}^{t}$ indicating the gradient obtained w.r.t. data batch $t$. Let $\bm{\bar g}_{\theta}^{t}= [(\bm{\bar g}_{\bm{\theta}_1}^{t})^{T}, (\bm{\bar g}_{\bm{\theta}_2}^{t})^{T}, ..., (\bm{\bar g}_{\bm{\theta}_\mathcal{L}}^{t})^{T}]^{T} \in \mathbb{R}^{\Sigma_{i=1}^{\mathcal{L}}n_i}$, which is obtained by
\begin{align}\label{eq_g_batch}
\bm{\bar g}_{\bm{\theta}_{l}}^{t} = \frac{\partial f(\bm{\theta}^{t})}{\partial\bm{\theta}_{l}^{t}}.
\end{align}  
If the dataset is large, the stochastic gradient descent (SGD) is often used as an alternative: 
\begin{align}\label{eq_sgd_g}
\bm{g}_{\bm{\theta}_{l}}^{t} = \frac{\partial f_{\bm{x}_{t}}(\bm{\theta}^{t})}{\partial\bm{\theta}_{l}^{t}}
\end{align}
where $\bm{x}_{t}$ is the $t^{\text{th}}$ mini-batch drawn from the dataset $\bm{x}$. We remove the bar ``$\ \bar{} \ $'' on $\bm{g}$ to tell the difference from \eqref{eq_g_batch}. Accordingly, the network weights can be updated through
\begin{align}\label{eq_sgd_update}
\bm{\theta}_{l}^{t+1} = \bm{\theta}_{l}^{t} - \gamma_{t}\bm{g}_{\bm{\theta}_{l}}^{t}, \ l=1,...,\mathcal{L}.
\end{align}  
Assume that each sample is randomly drawn with a uniform distribution. Then the gradient is unbiased:
\begin{align}\label{eq_expectation}
\mathbb{E}_{\bm{x}}\{\bm{g}_{\bm{\theta}_l}^{t}\} = \bm{\bar g}_{\bm{\theta}_l}^{t}
\end{align}
where the expectation $\mathbb{E}_{\bm{x}}$ is taken w.r.t. the random variable that draws $\bm{x}_{t}$ from the dataset.

To obtain the gradient vectors, the BP technique is used. We can calculate the gradients at layer $l$ using the gradients back-propagated from layers $j$ and $i$ ($l<j<i$) as follows:
\begin{align}\label{eq_activation_chain1}
\bm{g}_{\bm{\theta}_l}^{t} =\frac{\partial f_{\bm{x}_{t}}(\bm{\theta}^{t})}{\partial\bm{\theta}_{l}^{t}} = \frac{\partial \bm{z}_{j}^{t}}{\partial \bm{\theta}_{l}^{t}}\frac{\partial f_{\bm{x}_{t}}(\bm{\theta}^{t})}{\partial\bm{z}_{j}^{t}} = \frac{\partial \bm{z}_{j}^{t}}{\partial \bm{\theta}_{l}^{t}}\bm{g}_{\bm{z}_j}^{t}
\end{align}
where 
\begin{align}\label{eq_activation_chain2}
\bm{g}_{\bm{z}_j}^{t} = \frac{\partial f_{\bm{x}_{t}}(\bm{\theta}^{t})}{\partial\bm{z}_{j}^{t}} = \frac{\partial \bm{z}_{i}^{t}}{\partial\bm{z}_{j}^{t}}\frac{\partial f_{\bm{x}_{t}}(\bm{\theta}^{t})}{\partial\bm{z}_{i}^{t}} = \frac{\partial \bm{z}_{i}^{t}}{\partial\bm{z}_{j}^{t}}\bm{g}_{\bm{z}_i}^{t}.
\end{align}
Here we introduce $\bm{g}_{\bm{z}_j}^{t}$---the gradient vector w.r.t. activation $\bm{z}_j$---because it travels trough modules for communication in our ADL. Formulas \eqref{eq_activation_chain1} and \eqref{eq_activation_chain2} indicate that $\bm{g}_{\bm{\theta}_l}^{t}$ is obtained based on $\bm{g}_{\bm{z}_j}^{t}$ and $\bm{g}_{\bm{z}_i}^{t}$. That is, the gradient is not accessible before the forward pass is conducted and all the dependent gradients are obtained, which is known as the \textit{backward locking}. On the other hand, we cannot update the weights before every layers finishes its forward pass, which is recognized as the \textit{update locking}.

\subsection{Learning with Gradient Accumulation (GA)}
The GA technique has frequently been used to increase the mini-batch size for training networks on devices with a relatively limited memory setting. The gradients obtained based on several mini-batches are accumulated before they are finally applied to update the network.

To describe the training development involving the GA technique, we introduce an \textit{update index} $s$, and a \textit{wrapped batch index} $U_{s}$ w.r.t. the original batch index $t$. We use the update index $s$ to indicate the $s^{\text{th}}$ parameter update of the network. It is connected to $U_{s}$ in the way of $U_{s}=Ms$ given $M$ GA steps. Due to the GA technique, the network parameters remain unchanged for $M$ steps, i.e., $\bm{\theta}_{l}^{U_{s}}=$$\bm{\theta}_{l}^{U_{s}+1}=\dots=$$\bm{\theta}_{l}^{U_{s}+M-1}$. Inversely, we can tell the update index from a batch index $t$ by
\begin{align}\label{eq_t_to_s}
s=\floor{{t}/{M}}
\end{align}
where $\floor{x} = \text{max}\{n\in\mathbb{Z}|n \le x\}$ is the \textit{floor} operator. That is, when the network is processing the $t^{\text{th}}$ mini-batch of data, the network has been updated for $s$ times based on \eqref{eq_t_to_s}.

Assume that the gradients w.r.t. batch indexes $t = U_{s},U_{s}+1,\dots,U_{s}+M-1$ are accumulated. Accordingly, these gradients are obtained through
\begin{align}\label{eq_ac_g}
\bm{g}_{\bm{\theta}_{l}}^{t} = \frac{\partial f_{\bm{x}_{t}}(\bm{\theta}^{t})}{\partial\bm{\theta}_{l}^{U_{s}}}
\end{align}
where parameter $\bm{\theta}_{l}^{U_{s}}=\bm{\theta}_{l}^{U_{\floor{t/M}}}$ is adopted compared with \eqref{eq_sgd_g} to emphasize that the gradients w.r.t. to these data batches are obtained based on the same parameter. Using the GA technique, the weights are updated as follows:
\begin{align}\label{eq_bp_ac}
\bm{\theta}_{l}^{U_{s+1}} = \bm{\theta}_{l}^{U_{s}} - \gamma_{s}({1}/{M}){\textstyle\sum}_{j=0}^{M-1}\bm{g}_{\bm{\theta}_{l}}^{U_{s}+j}.
\end{align}

\subsection{Stale Gradient Effect}
Normally the network is updated with gradients obtained w.r.t. the current parameters. However, there are certain scenarios where the network has to update its parameters with gradients calculated based on ``older" parameters. This is called  the stale gradient effect or gradient staleness as the gradients are not up-to-date, and are therefore less accurate.

We define the \textit{level of staleness} (LoS) as the update index difference between the current parameter and the parameter used to calculate the stale gradient. That is, assume that a network is updated through
\begin{align}\label{eq_sgd_stale}
\bm{\theta}_{l}^{t+1} = \bm{\theta}_{l}^{t} - \gamma_{t}\bm{g}_{\bm{\theta}_l}^{t-d}, \ l=1,...,\mathcal{L}
\end{align}  
where $\bm{g}_{\bm{\theta}_l}^{t-d}= {\partial f_{\bm{x}_{t-d}}(\bm{\theta}^{t-d})}/{\partial\bm{\theta}_{l}^{t-d}}$. If the GA step is $M$, we could calculate the LoS through
\begin{align}\label{eq_LoS}
\mathrm{LoS} = \floor{t/M} - \floor{(t-d)/M}
\end{align}
indicating the current parameter is $\bm{\theta}_l^{U_{\floor{t/M}}}$ while the parameter used to calculate gradient $\bm{g}_{\bm{\theta}_l}^{t-d}$ is $\bm{\theta}_l^{U_{\floor{(t-d)/M}}}$.

\begin{figure*}
	\centering
 	\includegraphics[width=0.9\linewidth]{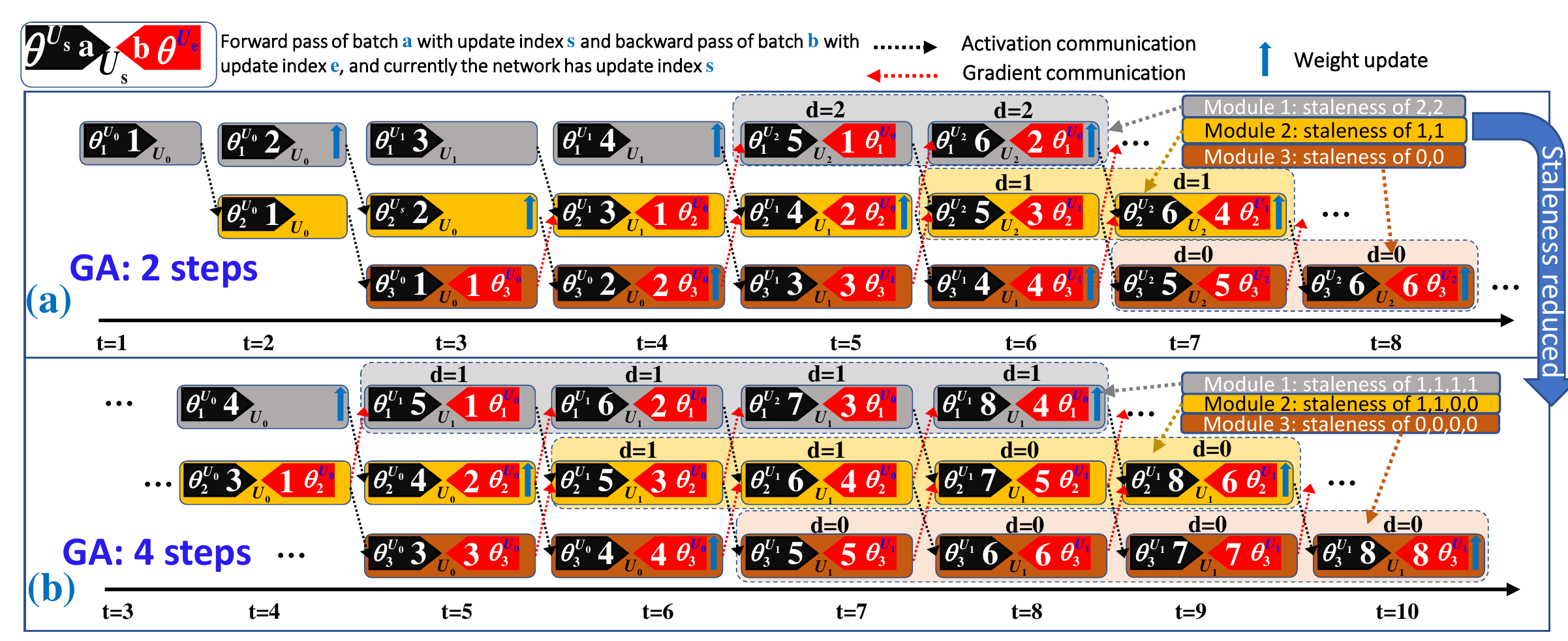}
	\caption{Training a 3-layer network by ADL with $K=3$ and GA steps of (a) $M=2$ and (b) $M=4$. Note that there is a batch index difference of $2(K-k)$ between the forward and backward pass. Gradient staleness is reduced with larger $M$.}
	\label{fig:macmain}
\end{figure*}
\section{The Proposed Method}
In this section, we show the algorithmic details of the proposed ADL, which include an asynchronous pipelining procedure to achieve model parallelism, and a GA technique to mitigate the stale gradient effect. In particular, we explicitly show how the GA could reduce the staleness.

Prior to our development, the network is split depth-wise into $K$ modules with a stack of layers in each module. That is, we split the set of the layer indices \{$1,\dots,\mathcal{L}$\} into \{$q(1), q(2), \dots, q(K)$\} where $q(k) = \{m_k, m_{k}+1,...,m_{k+1}-1\}$ denotes the layer indices in module $k$. This leads to possible notation changes as follows:
\begin{align*}
\resizebox{0.035\linewidth}{!}{$\bm{\theta}^{t}$} =& \resizebox{0.37\linewidth}{!}{$[(\bm{\theta}_{q(1)}^{t})^{T},..., (\bm{\theta}_{q(K)}^{t})^{T}]^{T}$}, \resizebox{0.5\linewidth}{!}{$\bm{\theta}_{q(k)}^{t} =[(\bm{\theta}_{m_k}^{t})^{T},..., (\bm{\theta}_{m_{k+1}-1}^{t})^{T}]^{T}$}\\
\resizebox{0.035\linewidth}{!}{$\bm{g}_{\bm{\theta}}^{t}$}=&\resizebox{0.37\linewidth}{!}{$  [(\bm{g}_{\bm{\theta}_{q(1)}}^{t})^{T},..., (\bm{g}_{\bm{\theta}_{q(K)}}^{t})^{T}]^{T}$}, \resizebox{0.5\linewidth}{!}{$\bm{g}_{\bm{\theta}_{q(k)}}^{t} = [(\bm{g}_{\bm{\theta}_{m_k}}^{t})^{T},..., (\bm{g}_{\bm{\theta}_{m_{k+1}-1}}^{t})^{T}]^{T}$}\\
\resizebox{0.035\linewidth}{!}{$\bm{\bar g}_{\theta}^{t}$} =& \resizebox{0.37\linewidth}{!}{$[(\bm{\bar g}_{\bm{\theta}_{q(1)}}^{t})^{T},..., (\bm{\bar g}_{\bm{\theta}_{q(K)}}^{t})^{T}]^{T}$}, \resizebox{0.5\linewidth}{!}{$\bm{\bar g}_{\bm{\theta}_{q(k)}}^{t} = [(\bm{\bar g}_{\bm{\theta}_{m_k}}^{t})^{T},..., (\bm{\bar  g}_{\bm{\theta}_{m_{k+1}-1}}^{t})^{T}]^{T}$}.
\end{align*} 
\subsection{Accumulated Decoupled Learning (ADL)}
We depict the proposed ADL with an example of training a 3-layer network with a split size $K=3$ in Fig. \ref{fig:macmain}(a) ($M=2$) and \ref{fig:macmain}(b) ($M=4$) respectively. As illustrated in the figures, at every iteration, each module runs a forward and a backward pass. The forward pass is executed with a module input that comes from the output of the lower module at the previous instance. The backward pass calculates the gradients by resuming the BP using gradients inherited from the upper module based on the ``older" data batches.  Note that all the split modules can be run in parallel due to asynchronism by processing data from different batches. Next, each module accumulates gradients for $M$ steps before the gradients are applied to update the network weights.

Assume that the weights of module $k$ ($k=1,\dots,K$) are  at update index $s$ with $\bm{\theta}_{q(k)}^{U_{s}}$. We detail the learning procedures in module $k$ to conduct update $s+1$ as follows.

\subsubsection{Forward Pass}
Module $k$ conducts the forward passes using data batches with indexes $U_{s}, U_{s}+1,$$\dots,$$U_{s}+M-1$. Let $j=0,1,\dots,M-1$.  In detail, we feed the module input $\bm{z}_{m_{k}-1}^{U_{s}+j}$ received from module\footnote{For $k=1$ the module input is the training data.} $k-1$ to generate activations in each layer, which are obtained w.r.t. the same parameter $\bm{\theta}_{q(k)}^{U_{s}}$. Next, we obtain the activation $\bm{z}_{m_{k+1}-1}^{U_{s}+j}$ at the end of this module, and send this activation to module $k+1$ (if any).

\subsubsection{Backward Pass}
During the backward pass, module $k$ resumes BP locally using the gradient\footnote{For $k=K$ the gradient is generated by the loss function.} $\bm{g}_{z_{m_{k+1}-1}}^{U_{s}+j-2(K-k)}$ received from module $k+1$. Note that the superscript $U_{s}+j-2(K-k)$ indicates that there are $2(K-k)$ steps of batch index delay w.r.t. the forward pass (see Fig. \ref{fig:macmain} for illustration). Accordingly, we calculate the gradients in each layer ($m_{k} \le l \le m_{k+1}-1$) within this module as follows:
\begin{align}\label{eq_adl_gra}
\bm{\hat g}_{\bm{\theta}_{l}}^{U_{s}+j} = \frac{\partial \bm{z}_{m_{k+1}-1}^{U_{s}+j-2(K-k)}}{\partial \bm{\theta}_{l}^{U_{\floor{(U_{s}+j-2(K-k))/M}}}}\bm{g}_{\bm{z}_{m_{k+1}-1}}^{U_{s}+j-2(K-k)}.
\end{align}
Note that \eqref{eq_adl_gra} is obtained w.r.t. $\bm{\theta}_{l}^{U_{\floor{(U_{s}+j-2(K-k))/M}}}$ with update index $\floor{(U_{s}+j-2(K-k))/M}$ instead of $s$. This is because the gradient is calculated based on the ``older" data batches, which can tell their corresponding update indexes from \eqref{eq_t_to_s}. At the end of the local BP, gradient 
$\bm{g}_{z_{m_{k}-1}}^{U_{s}+j-2(K-k)}$ w.r.t. the module input $z_{m_{k}-1}^{U_{s}+j-2(K-k)}$ is generated, which is then sent to module $k-1$ (if any).
\subsubsection{Update with Gradient Accumulation}
After obtaining the gradients using \eqref{eq_adl_gra}, the module is not updated immediately. Instead, we accumulate these gradients for $M$ steps before they are applied to update the module as follows:
\begin{align}\label{eq_adl_update}
\bm{\theta}_{l}^{U_{s+1}} &= \bm{\theta}_{l}^{U_{s}} - \gamma_{s}({1}/{M}){\textstyle\sum}_{j=0}^{M-1}\bm{\hat g}_{\bm{\theta}_{l}}^{U_{s}+j}.
\end{align}
We summarize the proposed ADL in Algorithm \ref{algo_adl}

Note that the above ADL is a lock-free decoupled technique. Firstly, the global BP is cast into local BPs in each module running in parallel, which removes the backward locking. Secondly, the split modules adopt training data from different batches so that the forward passes can be executed without waiting for the data from the lower layers. This tackles the forward locking. Finally, each module is updated immediately without waiting for other modules to complete their forward passes, hence addressing the update locking.

\begin{algorithm}
	\SetAlgoLined
	{\small Split the network into $K$ modules};\\
	\For{each iteration}{\small
		\For{$k \leftarrow 1$ \KwTo $K$ {\bf(Parallel)}}{\small
			\textbf{\textit{Forward pass}}: generate the activations with module input (e.g., $\bm{z}_{m_{k}-1}^{U_{s}+j}$), and send the module output (e.g., $\bm{z}_{m_{k+1}-1}^{U_{s}+j}$) to module $k+1$ (if any);\\
			\textbf{\textit{Backward pass}}: using gradient (e.g., $\bm{g}_{z_{m_{k+1}-1}}^{U_{s}+j-2(K-k)}$) received from module $k+1$ to calculate the gradients in each layer following \eqref{eq_adl_gra}, and send the gradient w.r.t. the module input (e.g.,  $\bm{g}_{z_{m_{k}-1}}^{U_{s}+j-2(K-k)}$) to module $k-1$ (if any);\\
			\textbf{\textit{Update}}: \If{accumulated $M$ steps of gradients}{
				Update the module using \eqref{eq_adl_update};
			}
		}
	}
	\caption{The proposed ADL}
	\label{algo_adl}
\end{algorithm}

\subsection{Impact of Gradient Accumulation}
Indicated by \eqref{eq_adl_gra}, the gradients are obtained based on $\bm{\theta}^{U_{\floor{(U_{s}+j-2(K-k))/M}}}$ while the parameter state is $\bm{\theta}^{U_{s}}$. Therefore, according to \eqref{eq_LoS}, the LoS for module $k$ is shown as follows ($j=0,1,\dots,M-1$):
\begin{align}\label{eq_adl_LoS}
d_{k,j} = s - \floor{(U_{s}+j-2(K-k))/M}.
\end{align}
For instance, as shown in Fig. \ref{fig:macmain}(b), with $M=4$ module $2$ updates its parameters using gradients with staleness of $d_{2,0}=1$, $d_{2,1}=1$, $d_{2,2}=0$, and $d_{2,3}=0$. According to \eqref{eq_adl_LoS}, the range of the staleness is
\begin{align}
0\le d_{k,j}\le 2(K-k)
\end{align}
with the minimum $d_{k,j}$ reached for $j-2(K-k)>0$, and the maximum $d_{k,j}=2(K-k)$ obtained at $M=1$ indicating no GA involved. For convenience, we adopt the \textit{averaged LoS}:
\begin{align}\label{eq_avg_LoS}
{\bar d}_{k} = ({1}/{M})\textstyle\sum_{j=0}^{M-1}d_{k,j}
\end{align}
to evaluate the staleness in the proposed ADL. As an example,
Fig. \ref{fig:los} shows the averaged LoS w.r.t. the accumulation step $M$ in the first module with $K=8$, where the gradient staleness is shown to reduce with increasing $M$.
\begin{figure}
	\centering
	\includegraphics[width=0.88\linewidth]{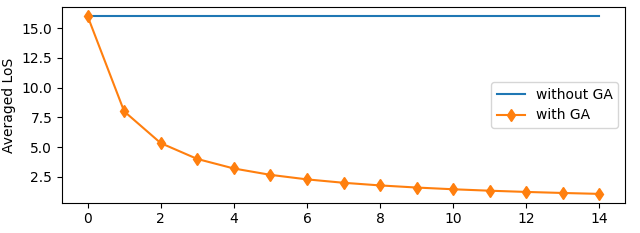}
	\caption{The averaged LoS w.r.t. the accumulation step $M$.}
	\label{fig:los}
\end{figure}

Large $M$ leads to lower gradient staleness, but does not necessarily guarantee improvement in network generalization. This is because larger accumulation step also indicate that the overall mini-batch size is increased, which could weaken the network's ability to generalize \cite{keskar2016large}. As a result, $M$ is an additional hyperparameter that handles the delicate balance between the stale gradient effect and the generalization. In fact, we do not need large $M$ to help the optimization because small $M$ has a rather significant impact on the staleness reduction. As shown in Fig. \ref{fig:los}, with $M=4$ the averaged LoS is already reduced by $75\%$ (from $16$ to $4$). Our experiments later also reveal that small $M$ ($2\le M\le 4$) works rather well.

In addition, according to \eqref{eq_adl_LoS}, we can unpack \eqref{eq_adl_gra} as
\begin{align}\nonumber
&\bm{\hat g}_{\bm{\theta}_{l}^{U_{s}}}^{U_{s}+j} = \frac{\partial \bm{z}_{m_{k+1}-1}^{U_{s}+j-2(K-k)}}{\partial \bm{\theta}_{l}^{U_{s-d_{k,j}}}}\frac{\partial f_{\bm{x}_{U_{s}+j-2(K-k)}}(\bm{\theta}^{U_{s-d_{k,j}}})}{\partial \bm{z}_{m_{k+1}-1}^{U_{s}+j-2(K-k)}}\\\label{eq_adl_fg_unpacked}
&=\frac{\partial f_{\bm{x}_{U_{s}+j-2(K-k)}}(\bm{\theta}^{U_{s-d_{k,j}}})}{\partial \bm{\theta}_{l}^{U_{s-d_{k,j}}}}=\bm{g}_{\bm{\theta}_{l}}^{U_{s}+j-2(K-k)}
\end{align}
and then rewrites \eqref{eq_adl_update} as
\begin{align}\label{eq_adl}
\bm{\theta}_{l}^{U_{s+1}} = \bm{\theta}_{l}^{U_{s}} - \gamma_{s}({1}/{M}){\textstyle\sum}_{j=0}^{M-1}\bm{g}_{\bm{\theta}_{l}}^{U_{s}+j-2(K-k)}.
\end{align}
That is, the proposed ADL accumulates gradients that are $2(K-k)$ steps ``older", while each of these accumulated gradients admits an LoS of $d_{k,j}$ as shown in \eqref{eq_adl_LoS}.

\section{Convergence Analysis}
In this section, we conduct convergence analysis of the proposed method. The analysis shows that the ADL can converge to critical points based on the following assumptions.

\begin{assumption}\label{ass_continous}
	Lipschitz continuity of gradients for loss functions $f(\bm{\theta})$, which means $\exists L\in\mathbb{R}^{+}$ such that:
	\begin{align}
	||\bm{\bar g}_{\bm{\theta}_{l}}^{U_{\alpha}} - \bm{\bar g}_{\bm{\theta}_{l}}^{U_{\beta}}||_2\le L||\bm{\theta}_{l}^{U_{\alpha}} - \bm{\theta}_{l}^{U_{\beta}}||_2
	\end{align}
	where $\left\lVert.\right\lVert_{2}$ is an $l_{2}$-norm operator. This also leads to
	\begin{align*}
	\resizebox{1\linewidth}{!}{$||\bm{\bar g}_{\bm{\theta}_{q(k)}}^{U_{\alpha}} - \bm{\bar g}_{\bm{\theta}_{q(k)}}^{U_{\beta}}||_2\le L||\bm{\theta}_{q(k)}^{U_{\alpha}} - \bm{\theta}_{q(k)}^{U_{\beta}}||_2, \ ||\bm{\bar g}_{\bm{\theta}}^{U_{\alpha}} - \bm{\bar g}_{\bm{\theta}}^{U_{\beta}}||_2\le L||\bm{\theta}^{U_{\alpha}} - \bm{\theta}^{U_{\beta}}||_2$}.
	\end{align*}
\end{assumption}

\begin{assumption}\label{ass_gradient_bound}
	Bounded variance of the stochastic gradient, which means that $\forall s$, $\exists A > 0$ such that:
	\begin{align}
	\resizebox{0.87\linewidth}{!}{$||\bm{g}_{\bm{\theta}_{l}}^{U_{s}}||_2^2\le A, \ \text{which leads to } ||\bm{g}_{\bm{\theta}_{q(k)}}^{U_{s}}||_2^2\le A, \ ||\bm{g}_{\bm{\theta}}^{U_{s}}||_2^2\le A$}.
	\end{align}
\end{assumption}

Assumptions \ref{ass_continous} and \ref{ass_gradient_bound} are commonly made for convergence analysis in neural networks (see \cite{bottou2018optimization,huo2018decoupled}). In particular, these assumptions do not assume convexity of function $f$.

\begin{thm}\label{thm_convergence}
	Let Assumptions \ref{ass_continous} and \ref{ass_gradient_bound} hold. Suppose that the learning rate is non-increasing and $L\gamma_{s}\le1$. The proposed ADL has the following lower bound:
	\begin{align}\label{eq_convergence}
	\resizebox{0.88\linewidth}{!}{$
		\mathbb{E}_{\bm{x}}\{f(\bm{\theta}^{U_{s+1}})\} - f(\bm{\theta}^{U_{s}}) \le - \frac{\gamma_{s}}{2}\lVert\bm{\bar g}_{\bm{\theta}}^{U_{s}}\lVert_2^2+ \gamma_{s}^{2}{AL}(1 + (1/M){\textstyle\sum}_{k=1}^{K}{\bar d}_{k})/M.$}
	\end{align}
\end{thm}
\begin{proof}
	See supplementary material A.
\end{proof}

Theorem \ref{thm_convergence} gives an important indication for convergence. If the RHS of \eqref{eq_convergence} is negative, i.e., 
\begin{align*}
	\resizebox{0.88\linewidth}{!}{$\gamma_{s} < \text{min}\left\{{1}/{L}, \ {M}\lVert\bm{\bar g}_{\bm{\theta}}^{U_{s}}\lVert_2^2/({2AL}(1 + (1/M){\textstyle\sum}_{k=1}^{K}{\bar d}_{k}))\right\}$},
\end{align*}
the expected loss $\mathbb{E}_{\bm{x}}\{f(\bm{\theta}^{U_{s+1}})\}$ would decrease. We further give the convergence evidence in the following theorems.

\begin{thm}\label{thm_avg_gra}
	Suppose Assumptions \ref{ass_continous} and \ref{ass_gradient_bound} hold, and the learning rate is non-increasing as well as satisfies $L\gamma_{s}\le1$. Let $\bm{\theta}^{*}$ be the global minimizer and $\mathbb{T}_{S} = {\textstyle\sum}_{s=0}^{S-1}\gamma_{s}$ where $S$ indicates the network will be updated $S$ times. Then
	\begin{align}\nonumber
	&({1}/{\mathbb{T}_{S}}){\textstyle\sum}_{s=0}^{S-1}\gamma_{s}\mathbb{E}\{||\bm{\bar g}_{\bm{\theta}}^{U_{s}}||_2^2\} \le {2(f(\bm{\theta}^{0}) - f(\bm{\theta}^{*}))}/{\mathbb{T}_{S}} \\\label{eq_bound_weighted_avg_gra}
	&+ ({2{AL}(1 + (1/M){\textstyle\sum}_{k=1}^{K}{\bar d}_{k}){\textstyle\sum}_{s=0}^{S-1}\gamma_{s}^{2} })/({{M}\mathbb{T}_{S}}).
	\end{align}
\end{thm}
\begin{proof}
	See supplementary material B.
\end{proof}
We use \textit{ergodic convergence} as the metric to evaluate the convergence, which is commonly adopted for convergence analysis in non-convex optimization (see \cite{bottou2018optimization,huo2018decoupled,lian2015asynchronous}). The lower bound in Theorem \ref{thm_avg_gra} indicates that, for a randomly selected $\mathfrak{q}$ from $\{0,1,\dots,S-1\}$ with probability $\{\gamma_{\mathfrak{q}}/\mathbb{T}_{S}\}$, $\mathbb{E}\{||\bm{\bar g}_{\bm{\theta}}^{U_{s}}||_2^2\}$ is bounded by the RHS of \eqref{eq_bound_weighted_avg_gra}. More importantly, a larger $M$ leads to a smaller lower bound in \eqref{eq_bound_weighted_avg_gra} because the ${\bar d}_{k}$ decreases, and thus benefits the convergence. Another observation is that larger split size $K$ hinders the convergence as ${\textstyle\sum}_{k=1}^{K}{\bar d}_{k}$ increases. These observations are consistent to our understanding that the GA helps the optimization by mitigating staleness, and splitting the network into more modules is harmful.

\begin{corollary}\label{corollary_lr}
	If $\gamma_{s}$ further satisfies $\lim_{S\to\infty}\mathbb{T}_{S}$$=\infty$ and $\lim_{S\to\infty}\sum_{s=0}^{S-1}\gamma_{s}^{2}<$$\infty$, the RHS of \eqref{eq_bound_weighted_avg_gra} converges to 0.
\end{corollary}
According to Corollary \ref{corollary_lr}, by properly scheduling the learning rate, the lower bound for the expected gradient would converge to $0$, i.e.,  $\lim_{S\to\infty}\mathbb{E}\{||\bm{\bar g}_{\bm{\theta}}^{U_{s}}||_2^2\}=0$. That is, the proposed ADL can converge to critical points. Alternatively, the convergence can be revealed by setting a constant learning rate as indicated in the following theorem.

\begin{thm}\label{thm_constant_lr}
	Let Assumptions \ref{ass_continous} and \ref{ass_gradient_bound} hold. Suppose the learning rate is set as a constant:
	\begin{align*}
	\resizebox{0.68\linewidth}{!}{$
		\gamma = \epsilon\sqrt{M(f(\bm{\theta}^{0}) - f(\bm{\theta}^{*}))/\Big({SAL}(1+ {\textstyle\sum}_{k=1}^{K}{\bar d}_{k})}\Big)$}
	\end{align*}
	where $\epsilon$ is a scaling factor such that $L\gamma\le1$. Let $\bm{\theta}^{*}$ be the global minimizer. Then we have
	\begin{align}\label{eq_constant_lr}
	\resizebox{0.88\linewidth}{!}{$
		\underset{s\in\{0,1,\dots,S-1\}}{\mathrm{min}}\mathbb{E}\{\lVert\bm{\bar g}_{\bm{\theta}}^{U_{s}}\rVert_2^2\}\le \frac{(2 + 2\epsilon^{2})}{\epsilon}\sqrt{AL(f(\bm{\theta}^{0}) - f(\bm{\theta}^{*}))\Big( 1 + (1/M)\textstyle\sum_{k=1}^{K}{\bar d}_{k}\Big)/(MS)},$}
	\end{align}
	where the lower bound converges to 0 when $S\to \infty$.
\end{thm}
\begin{proof}
	See supplementary material C
\end{proof}

In summary, although the ADL attains model parallelism by adopting asynchronization, we show that our method can converge to critical points, and reveal how the convergence can be affected by the GA step $M$ and the split size $K$.

\section{Experiments}\label{section_ex}
In this section, we conduct classification tasks on the well-known CIFAR-10 \cite{krizhevsky2009learning} and ImageNet 2012 \cite{russakovsky2015imagenet} datasets to evaluate the classification and acceleration performance with various split sizes $K$. We compare our method with several state-of-the-arts, including DDG \cite{huo2018decoupled}, FR \cite{huo2018training}, DGL \cite{belilovsky2019decoupled}, Gpipe \cite{huang2018gpipe}, and DSP \cite{xu2020acceleration}, as well as BP \cite{werbos1974beyond}.
\begin{table}
	\centering
	\caption{Testing errors of the compared methods on (a) CIFAR-10 and (b) ImageNet (Top1/Top5).}
	\label{table_c10}
	\resizebox{1\linewidth}{!}{%
		{\large \bf (a)  }
		\begin{tabular}{lllccccc}
			\toprule[0.3mm]
			&\multicolumn{1}{c}{Architecture}&\multicolumn{1}{c}{\text{BP}} &\text{DDG}  &\text{DGL}  & \text{FR} &\text{DSP} &ADL\\ 
			\hline
			\multirow{7}{*}{$K\le 4$}
			&ResNet-56 ($K=2$)&6.19\% &6.63\%&6.77\%&\textbf{6.07}\%&-&\textbf{6.07}\%($M=2$)\\
			&ResNet-56 ($K=3$)&6.19\%&6.50\%&8.88\%&6.33\%&-&\textbf{6.09}\%($M=4$)\\
			&ResNet-56 ($K=4$)&6.19\%&6.61\%&9.65\%& 6.48\% &-&\textbf{6.16}\%($M=4$)\\
			&ResNet-18 ($K=2$)&4.87\%&5.00\%&5.21\%&\textbf{4.80}\%&-&4.82\%($M=2$)\\
			&ResNet-110 ($K=2$)&5.79\% &6.26\%&6.26\%&5.76\%&-&\textbf{5.70}\%($M=2$)\\
			&ResNet-98 ($K=4$)&6.01\%&-&-&-&6.59\%&\textbf{5.90}\%($M=3$)\\
			&ResNet-164 ($K=4$)&\textbf{5.36}\%&-&-&-&5.58\%&5.45\%($M=2$)\\
			\hline
			\multirow{4}{*}{$K> 4$}
			&ResNet-56 ($K=8$)&6.19\%&-&-&-&-&\textbf{6.18}\%($M=4$)\\
			&ResNet-18 ($K=8$)&\textbf{4.87}\%&-&-&-&-&4.92\%($M=4$)\\
			&ResNet-110 ($K=8$)&\textbf{5.79}\%&-&-&-&-&5.80\%($M=4$)\\
			&ResNet-164 ($K=10$)&\textbf{5.36}\%&-&-&-&-&5.52\%($M=2$)\\
			\hline
			\toprule[0.3mm]
	\end{tabular} }
	\resizebox{1\linewidth}{!}{%
		{\large \bf (b)  }
		\begin{tabular}{clccccccc}
			\toprule[0.3mm]
			&\multicolumn{1}{c}{Architecture}&\multicolumn{1}{c}{\text{BP}}&FR&DSP&\multicolumn{1}{c}{ADL}\\
			\hline
			\multirow{4}{*}{$K\le 4$}
			&ResNet-18 ($K=3$) & 29.79\%/10.92\%&31.16\%/-&31.15\%/-&\textbf{29.51}\%($M=2$)/\textbf{10.41}\%($M=2$)\\
			&ResNet-18 ($K=4$)& 29.79\%/10.92\%&-/-&-/-&\textbf{29.64}\%($M=4$)/\textbf{10.56}\%($M=4$)\\
			&ResNet-50 ($K=3$) & \textbf{23.65}\%/{7.13\%}&25.53\%/-&25.09\%/-&{23.92}\%($M=2$)/\textbf{7.07}\%($M=2$)\\
			&ResNet-50 ($K=4$) & 23.65\%/\textbf{7.13\%}&-/-&-/-&\textbf{23.37}\%($M=8$)/7.44\%($M=8$)\\
			\hline
			\multirow{4}{*}{$K> 4$}
			&ResNet-18 ($K=8$) & 29.79\%/10.92\%&-/-&-/-&\textbf{29.75}\%($M=4$)/\textbf{10.55}\%($M=4$)\\
			&ResNet-18 ($K=10$, max.) & \textbf{29.79}\%/10.92\%&-/-&-/-&29.84\%($M=4$)/\textbf{10.76}\%($M=4$)\\
			&SE-ResNet-18 ($K=8$) & 29.09\%/\textbf{9.89}\%&-/-&-/-&\textbf{29.01}\%($M=4$)/10.14\%($M=4$)\\
			&SE-ResNet-18 ($K=10$, max.) & 29.09\%/\textbf{9.89}\%&-/-&-/-&\textbf{29.07}\%($M=2$)/10.31\%($M=2$)\\
			\hline
			\toprule[0.3mm]
	\end{tabular} }
\end{table}

\begin{figure*}
	\centering
	\includegraphics[width=0.95  \linewidth]{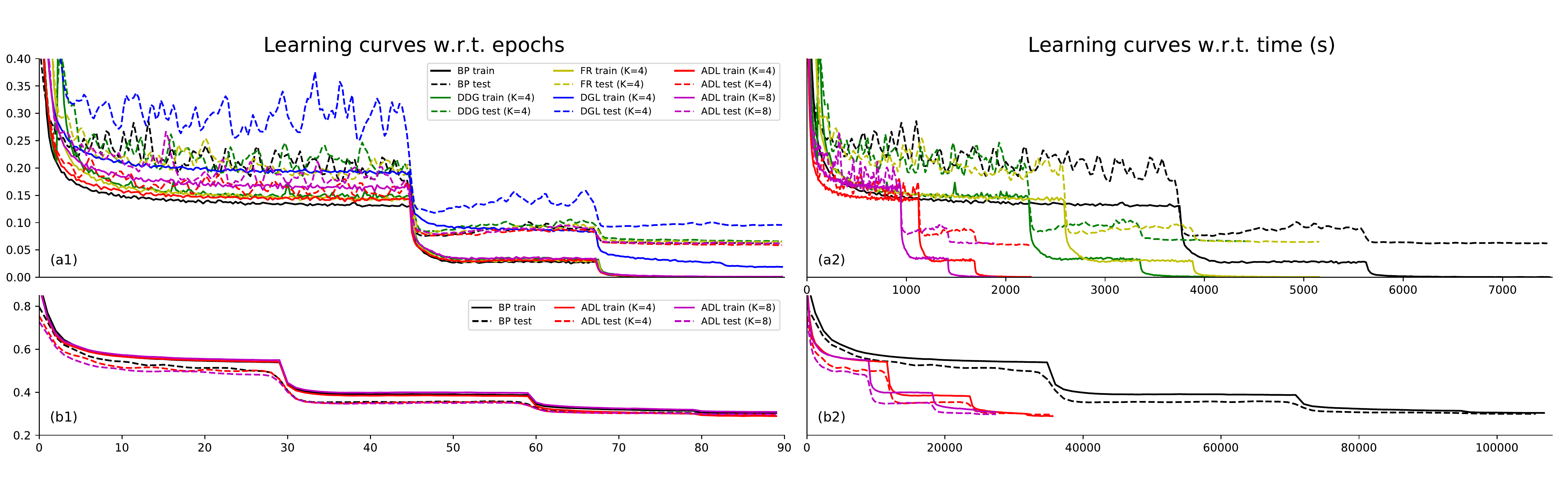}
	\caption{Learning curves (errors rates) for training (a) ResNet-56 on CIFAR-10, and (b) ResNet-18 on ImageNet.}
	\label{fig:lcs}
\end{figure*}
\noindent\textbf{Implementation Details}: The experiments are performed with Pytorch \cite{NEURIPS2019_9015} where we pre-process the datasets using standard data augmentation (i.e., random cropping, random horizontal flip and normalizing \cite{he2016deep}). The SGD optimizer with a momentum of 0.9 is adopted, and the models are trained using a batch size $b=32$. For a GA step of $M$, the initial learning rate is set at $0.1\times bM/256$. In addition, we adopt the gradual warm-up in \cite{goyal2017accurate} for 3 epochs. The testing errors of all the experiments are reported at the \textit{last epoch} by the median of 3 runs. For CIFAR-10, the weight decay is set at $5\times 10^{-4}$, and the models are trained for 300 epochs with the learning rate divided by 10 at 150, 225 and 275 epochs. For ImageNet, a 224$\times$×224 crop is randomly sampled, and the weight decay is set at $1\times 10^{-4}$. We train the networks for 90 epochs, and divide the learning rate by 10 at 30, 60, and 80 epochs.

\subsection{Generalization Performance}
\subsubsection{CIFAR-10} 
The CIFAR-10 dataset includes 32x32 color images with 10 classes, and has 50000 and 10000 samples for training and testing respectively. We train several architectures from ResNet \cite{he2016deep} and the classification results are shown in Table \ref{table_c10}(a). For a small split size (i.e., $K\le4$), in general the proposed ADL outperforms the compared methods, including the global BP. On the other hand, the ADL can push the split size from $K=4$ (maximum split size reported in the prior arts) up to $K=10$ while maintaining a comparable classification performance to the BP baseline. In particular, training ResNet-56 by splitting it into 8 modules even slightly outperforms the global BP.

\subsubsection{ImageNet}
The ImageNet dataset contains 1000 classes, and includes 1.28 million and 50000 images of various sizes for training and testing. We train several architectures from ResNet \cite{he2016deep} and SENet \cite{hu2018squeeze}, and report both Top1 and Top5 error rates. As shown in Table \ref{table_c10}(b), in general the proposed ADL outperforms (according to the Top1 results) the compared methods with either small ($K\le 4$) or large split size ($K>4$). We show that the ADL can maximally split the networks (e.g., ResNet-18 and SE-ResNet18) into 10 modules---with each module containing only one layer or one residual block---to facilitate model parallelism without compromising the generalization. These results are more promising than those from the CIFAR-10 experiments. It might be because that training on ImageNet is less sensitive to changes of batch size or batch normalization, which is evidenced by various distributed learning methods \cite{goyal2017accurate}.

To observe the convergence details, we also depict several examples of learning curves to show the training dynamic of the compared methods. Fig. \ref{fig:lcs}(a1) and \ref{fig:lcs}(b1) give the learning curves w.r.t. epochs, which show that the ADL converges smoothly in a similar way to the BP's. Fig. \ref{fig:lcs}(a2) and \ref{fig:lcs}(b2) show the learning curves w.r.t. wall time. The proposed ADL achieves the highest speedup among the compared methods, which is expected as the proposed method is a lock-free decoupled learning method unlike the DDG and FR that only tackle the backward locking.
\subsubsection{A Simple Ablation Study}
Here we conduct a very simple ablation study to show the significance of the GA technique to the proposed ADL. We train networks on CIFAR-10 using the ADL with ($M>1$) and without ($M=1$) the GA technique. As indicated in Table \ref{table_ab}, networks trained without the GA technique tend to give worse performance or even diverge due to strong gradient staleness, while the ADL with GA could give comparable results to the BP baselines. This simple study provides an empirical evidence for the necessity to include the GA technique in the ADL.

\begin{table}
	\centering
	\caption{Results for ablation study.}
	\label{table_ab}
	
	\resizebox{0.8\linewidth}{!}{%
		\begin{tabular}{cccc}
			\toprule[0.3mm]
			Architecture&BP&ADL with GA&ADL without GA\\ 
			\hline
			ResNet-18 ($K=8$) &4.87\%&4.92\%($M=4$)&5.50\%($M=1$)\\
			ResNet-56 ($K=8$) &6.19\%&6.18\%($M=2$)&div. ($M=1$)\\
			\hline
			\toprule[0.3mm]
	\end{tabular} }
\end{table}

\subsection{Acceleration Performance}
We show the acceleration performance of the proposed ADL on a server with Tesla V100 GPUs through training ResNet-101 on ImageNet, and ResNet-1202 on CIFAR-10 with $K=4,8$. Batch size is adjusted to maximize the training speed, and the network split locations are tuned to distribute the workload as evenly as possible. 

As shown in Table \ref{table_acceleration}, the proposed ADL achieves the best acceleration in the learning examples due to fully addressing the locking problem in BP. For training ResNet-101 ($K=4$) on ImageNet, the ADL is able to give a \textbf{3.32}$\times$ speedup, which is the fastest among methods of decoupled learning. We notice that the acceleration is not well-delivered after scaling the split size to $8$, which only achieves a speedup of \textbf{3.95}$\times$. This is due to the imbalanced workload allocation among different modules. It is an inevitable issue in methods that require a depth-wise partition without a custom design to evenly distribute the computation into each worker. The imbalance becomes more observable if larger $K$ is considered. Such imbalance issue can be verified in the ResNet-1202 example in Table \ref{table_acceleration}, where the acceleration is relatively more significant (e.g., \textbf{3.60}$\times$ and \textbf{6.30}$\times$ for $K=4,8$ respectively). The improvement over the ResNet-101 case is because ResNet-1202 has more layers, and hence easily leads to a more balanced workload partition. 

\begin{table}
	\caption{Examples of speedups (over BP) in training ResNet-101 (ImageNet) and ResNet1202 (CIFAR-10).}
	\label{table_acceleration}
	\centering
	
	\resizebox{1\linewidth}{!}{%
		\begin{tabular}{lccccccc}
			\toprule[0.3mm]
			\hline
			&BP&DDG&FR&Gpipe&DSP&ADL\\
			\hline
			ResNet-101 (K=4)&1$\times$&1.68$\times$&1.45$\times$&2.20$\times$&2.70$\times$&\textbf{3.32}$\times$\\
			ResNet-101 (K=8)&1$\times$&-&-&3.00$\times$&-&\textbf{3.95}$\times$\\
			ResNet-1202 (K=4)&1$\times$&-&-&-&-&\textbf{3.60}$\times$\\
			ResNet-1202 (K=8)&1$\times$&-&-&-&-&\textbf{6.30}$\times$\\
			\hline
			\bottomrule[0.3mm]
	\end{tabular} }
\end{table}

In summary, the proposed ADL gives comparable or better results in classification tasks for $K\le 4$ compared with various methods, and is shown to perform robustly and accurately for $K>4$ (up to $K=10$) where the current state-of-the-arts cannot reach. As a model parallelism tool, the ADL is the fastest among the compared methods.

\section{Conclusion}
In this paper, we proposed the accumulated decoupled learning (ADL) to address the inefficient BP lockings thereby achieving model parallelism. The proposed method incorporates the gradient accumulation technique, which mitigates the stale gradient effect that hinders the scaling ability of the decoupled learning. The mitigation has been demonstrated theoretically, and also evidenced empirically through the ablation study. Our convergence analysis has shown that the ADL can converge to critical points, i.e., the gradients converge to 0. The classification tasks conducted showed that the proposed ADL in general outperformed the state-of-the-art counterparts in terms of both accuracy and training acceleration.

\bibliographystyle{IEEEtran}
\bibliography{DG}

\newpage
\onecolumn
\section*{Supplementary material A: Proof of Theorem \ref{thm_convergence}}
\begin{proof}
	To simplify the notations, let $\bm{\mathfrak{g}}_{\bm{\theta}_{q(k)}}^{U_{s}^{'}} =$$ \frac{1}{M}\sum_{j=0}^{M-1}\bm{g}_{\bm{\theta}_{q(k)}}^{U_{s}+j-2(K-k)}$ and $\bm{\bar \mathfrak{g}}_{\bm{\theta}_{q(k)}}^{U_{s}^{'}} =$$ \frac{1}{M}\sum_{j=0}^{M-1}\bm{\bar g}_{\bm{\theta}_{q(k)}}^{U_{s}+j-2(K-k)}$.
	According to Assumption 1, the following inequality holds:
	\begin{align}\nonumber
	f(\bm{\theta}^{U_{s+1}}) \le& f(\bm{\theta}^{U_{s}}) + (\bm{\bar g}_{\bm{\theta}}^{U_{s}})^{T}(\bm{\theta}^{U_{s+1}} - \bm{\theta}^{U_{s}}) + \frac{L}{2}\Big\lVert\bm{\theta}^{U_{s+1}} - \bm{\theta}^{U_{s}}\Big\lVert_2^2\\\label{eq_continous}
	=&f(\bm{\theta}^{U_{s}}) - \gamma_{s}\sum\limits_{k=1}^{K}(\bm{\bar g}_{\bm{\theta}_{q(k)}}^{U_{s}})^{T}\bm{\mathfrak{g}}_{\bm{\theta}_{q(k)}}^{U_{s}^{'}} + \frac{L\gamma_{s}^{2}}{2}\sum\limits_{k=1}^{K}\Big\lVert\bm{\mathfrak{g}}_{\bm{\theta}_{q(k)}}^{U_{s}^{'}}\Big\lVert_2^{2}
	\end{align}
	which can be further developed such that
	\begin{align}\nonumber
	f(\bm{\theta}^{U_{s+1}}) &\le f(\bm{\theta}^{U_{s}})- \gamma_{s}\sum\limits_{k=1}^{K}(\bm{\bar g}_{\bm{\theta}_{q(k)}}^{U_{s}})^{T}(\bm{\mathfrak{g}}_{\bm{\theta}_{q(k)}}^{U_{s}^{'}} - \bm{\bar g}_{\bm{\theta}_{q(k)}}^{U_{s}} + \bm{\bar g}_{\bm{\theta}_{q(k)}}^{U_{s}})+ \frac{L\gamma_{s}^{2}}{2}\sum\limits_{k=1}^{K}\Big\lVert\bm{\mathfrak{g}}_{\bm{\theta}_{q(k)}}^{U_{s}^{'}} - \bm{\bar g}_{\bm{\theta}_{q(k)}}^{U_{s}} + \bm{\bar g}_{\bm{\theta}_{q(k)}}^{U_{s}}\Big\lVert_2^{2}\\\nonumber
	&=f(\bm{\theta}^{U_{s}}) - \gamma_{s}\sum\limits_{k=1}^{K}\Big\lVert \bm{\bar g}_{\bm{\theta}_{q(k)}}^{U_{s}}\Big\lVert_2^{2} -\gamma_{s}\sum\limits_{k=1}^{K}(\bm{\bar g}_{\bm{\theta}_{q(k)}}^{U_{s}})^{T}(\bm{\mathfrak{g}}_{\bm{\theta}_{q(k)}}^{U_{s}^{'}} - \bm{\bar g}_{\bm{\theta}_{q(k)}}^{U_{s}}) + \frac{L\gamma_{s}^{2}}{2}\sum\limits_{k=1}^{K}\Big\lVert \bm{\bar g}_{\bm{\theta}_{q(k)}}^{U_{s}}\Big\lVert_2^{2}\\\nonumber
	& + \frac{L\gamma_{s}^{2}}{2}\sum\limits_{k=1}^{K}\Big\lVert\bm{\mathfrak{g}}_{\bm{\theta}_{q(k)}}^{U_{s}^{'}} - \bm{\bar g}_{\bm{\theta}_{q(k)}}^{U_{s}} \Big\lVert_2^{2} + L\gamma_{s}^{2}\sum\limits_{k=1}^{K}(\bm{\bar g}_{\bm{\theta}_{q(k)}}^{U_{s}})^{T}(\bm{\mathfrak{g}}_{\bm{\theta}_{q(k)}}^{U_{s}^{'}} - \bm{\bar g}_{\bm{\theta}_{q(k)}}^{U_{s}})\\\label{eq_appendix_3}
	&= f(\bm{\theta}^{U_{s}}) - (\gamma_{s} - \frac{L\gamma_{s}^{2}}{2})\sum\limits_{k=1}^{K}\Big\lVert\bm{\bar g}_{\bm{\theta}_{q(k)}}^{U_{s}}\Big\lVert_2^2 + \tilde{Q}_1 + \tilde{Q}_2
	\end{align}
	where
	\begin{align*}
	\tilde{Q}_{1} =  \frac{L\gamma_{s}^{2}}{2}\sum\limits_{k=1}^{K}\Big\lVert\bm{\mathfrak{g}}_{\bm{\theta}_{q(k)}}^{U_{s}^{'}} - \bm{\bar g}_{\bm{\theta}_{q(k)}}^{U_{s}} \Big\lVert_2^{2}, \  \tilde{Q}_{2} &= (L\gamma_{s}^{2} - \gamma_{s})\sum\limits_{k=1}^{K}(\bm{\bar g}_{\bm{\theta}_{q(k)}}^{U_{s}})^{T}(\bm{\mathfrak{g}}_{\bm{\theta}_{q(k)}}^{U_{s}^{'}} - \bm{\bar g}_{\bm{\theta}_{q(k)}}^{U_{s}}).
	\end{align*}
	The expectation of $\tilde{Q}_1$ is bounded by
	
	\begin{align*}
	\mathbb{E}_{\bm{x}}\{\tilde{Q}_1\} =& \frac{L\gamma_{s}^{2}}{2}\mathbb{E}_{\bm{x}}\{\sum\limits_{k=1}^{K}\Big\lVert\bm{\mathfrak{g}}_{\bm{\theta}_{q(k)}}^{U_{s}^{'}} - \bm{\bar g}_{\bm{\theta}_{q(k)}}^{U_{s}}\Big\lVert_{2}^{2}\} = \frac{L\gamma_{s}^{2}}{2}\mathbb{E}_{\bm{x}}\{\sum\limits_{k=1}^{K}\Big\lVert\bm{\mathfrak{g}}_{\bm{\theta}_{q(k)}}^{U_{s}^{'}} -\bm{\bar \mathfrak{g}}_{\bm{\theta}_{q(k)}}^{U_{s}^{'}} - \bm{\bar g}_{\bm{\theta}_{q(k)}}^{U_{s}} + \bm{\bar \mathfrak{g}}_{\bm{\theta}_{q(k)}}^{U_{s}^{'}} \Big\lVert_{2}^{2}\}\\
	\le & L\gamma_{s}^{2}\mathbb{E}_{\bm{x}}\{\sum\limits_{k=1}^{K}\Big\lVert\bm{\mathfrak{g}}_{\bm{\theta}_{q(k)}}^{U_{s}^{'}} -\bm{\bar \mathfrak{g}}_{\bm{\theta}_{q(k)}}^{U_{s}^{'}}\Big\lVert_{2}^{2}\} +   L\gamma_{s}^{2}\sum\limits_{k=1}^{K}\Big\lVert\bm{\bar \mathfrak{g}}_{\bm{\theta}_{q(k)}}^{U_{s}^{'}} - \bm{\bar g}_{\bm{\theta}_{q(k)}}^{U_{s}}\Big\lVert_{2}^{2}\\
	=& L\gamma_{s}^{2}\mathbb{E}_{\bm{x}}\{\Big\lVert\bm{\mathfrak{g}}_{\bm{\theta}}^{U_{s}^{'}} -\bm{\bar \mathfrak{g}}_{\bm{\theta}}^{U_{s}^{'}}\Big\lVert_{2}^{2}\} +   L\gamma_{s}^{2}\sum\limits_{k=1}^{K}\Big\lVert\bm{\bar \mathfrak{g}}_{\bm{\theta}_{q(k)}}^{U_{s}^{'}} - \bm{\bar g}_{\bm{\theta}_{q(k)}}^{U_{s}}\Big\lVert_{2}^{2}\\
	\le & L\gamma_{s}^{2}\mathbb{E}_{\bm{x}}\{\Big\lVert\bm{\mathfrak{g}}_{\bm{\theta}}^{U_{s}^{'}}\Big\lVert_{2}^{2}\} +   L\gamma_{s}^{2}\sum\limits_{k=1}^{K}\Big\lVert\bm{\bar \mathfrak{g}}_{\bm{\theta}_{q(k)}}^{U_{s}^{'}} - \bm{\bar g}_{\bm{\theta}_{q(k)}}^{U_{s}}\Big\lVert_{2}^{2}\\
	\le & L\gamma_{s}^{2}\frac{1}{M^{2}}\mathbb{E}_{\bm{x}}\{\sum_{j=0}^{M-1}\Big\lVert\bm{\bar g}_{\bm{\theta}_{q(k)}}^{U_{s}+j-2(K-k)}\Big\lVert_{2}^{2}\} +   L\gamma_{s}^{2}\sum\limits_{k=1}^{K}\Big\lVert\bm{\bar \mathfrak{g}}_{\bm{\theta}_{q(k)}}^{U_{s}^{'}} - \bm{\bar g}_{\bm{\theta}_{q(k)}}^{U_{s}}\Big\lVert_{2}^{2}\\
	\le & \frac{AL}{M}\gamma_{s}^{2} + L\gamma_{s}^{2}\sum\limits_{k=1}^{K}\Big\lVert\bm{\bar \mathfrak{g}}_{\bm{\theta}_{q(k)}}^{U_{s}^{'}} - \bm{\bar g}_{\bm{\theta}_{q(k)}}^{U_{s}}\Big\lVert_{2}^{2} = \frac{AL}{M}\gamma_{s}^{2} + L\gamma_{s}^{2}\tilde{P}_1
	\end{align*}
	where the first inequality follows from $\lVert\bm{x} + \bm{y}\lVert_{2}^{2}\le 2\lVert\bm{x}\lVert_2^2 + 2\lVert\bm{y}\lVert_2^2$. The second inequality is from $\mathbb{E}\{\lVert\epsilon - \mathbb{E}\{\epsilon\}\lVert_{2}^{2}\}\le\mathbb{E}\{\lVert\epsilon\lVert_{2}^{2}\} -  \lVert\mathbb{E}\{\epsilon\}\lVert_{2}^{2}\le \mathbb{E}\{\lVert\epsilon\lVert_{2}^{2}\}$ due to gradient unbiasedness (i.e., $\mathbb{E}_{\bm{x}}\{\bm{\mathfrak{g}}_{\bm{\theta}_{q(k)}}^{U_{s}^{'}}\} = \bm{\bar \mathfrak{g}}_{\bm{\theta}_{q(k)}}^{U_{s}^{'}}$). The last inequality follows from Assumption 2, and $\tilde{P}_1$ is bounded by 
	\begin{align*}
	\tilde{P}_1 =& \sum\limits_{k=1}^{K}\Big\lVert\bm{\bar \mathfrak{g}}_{\bm{\theta}_{q(k)}}^{U_{s}^{'}} - \bm{\bar g}_{\bm{\theta}_{q(k)}}^{U_{s}}\Big\lVert_{2}^{2}\le  \frac{1}{M^{2}}\sum\limits_{k=1}^{K}\sum\limits_{j=0}^{M-1}\Big\lVert\bm{\bar g}_{\bm{\theta}_{q(k)}}^{U_{s}+j-2(K-k)} - \bm{\bar g}_{\bm{\theta}_{q(k)}}^{U_{s}}\Big\lVert_{2}^{2}\\
	=& \frac{L^{2}}{M^{2}}\sum\limits_{k=1}^{K}\sum\limits_{j=0}^{M-1}\Big\lVert\bm{\theta}_{q(k)}^{U_{s}} - \bm{\theta}_{q(k)}^{U_{s-d_{k,j}}}\Big\lVert_{2}^{2}\\
	=& \frac{L^{2}}{M^{2}}\sum\limits_{k=1}^{K}\sum\limits_{j=0}^{M-1}\Big\lVert\sum\limits_{\alpha=\mathrm{max}\{0,s-d_{k,j}\}}^{s-1}(\bm{\theta}_{q(k)}^{U_{\alpha+1}} - \bm{\theta}_{q(k)}^{U_{\alpha}})\Big\lVert_{2}^{2} \le  \frac{L^{2}}{M^{2}}\sum\limits_{k=1}^{K}\sum\limits_{j=0}^{M-1}\sum\limits_{\alpha=\mathrm{max}\{0,s-d_{k,j}\}}^{s-1}\Big\lVert\bm{\theta}_{q(k)}^{U_{\alpha+1}} - \bm{\theta}_{q(k)}^{U_{\alpha}}\Big\lVert_{2}^{2}\\
	= & \frac{L^{2}}{M^{2}}\sum\limits_{k=1}^{K}\sum\limits_{j=0}^{M-1}\sum\limits_{\alpha=\mathrm{max}\{0,s-d_{k,j}\}}^{s-1}\gamma_{\alpha}^{2}\Big\lVert\bm{\mathfrak{g}}_{\bm{\theta}_{q(k)}}^{U_{s}^{'}}\Big\lVert_{2}^{2} \le \frac{L^{2}}{M^{2}}\sum\limits_{k=1}^{K}\sum\limits_{j=0}^{M-1}\sum\limits_{\alpha=\mathrm{max}\{0,s-d_{k,j}\}}^{s-1}\gamma_{\alpha}^{2}\frac{1}{M^{2}}\sum_{j=0}^{M-1}\Big\lVert\bm{g}_{\bm{\theta}_{q(k)}}^{U_{s}+j-2(K-k)}\Big\lVert_{2}^{2}\\
	\le & \frac{AL^{2}}{M^{2}}\sum\limits_{k=1}^{K}\sum\limits_{j=0}^{M-1}\sum\limits_{\alpha=\mathrm{max}\{0,s-d_{k,j}\}}^{s-1}\gamma_{\alpha}^{2}
	\le  \gamma_{s}^{2}\frac{AL^{2}}{M^{2}}\sum\limits_{k=1}^{K}\sum\limits_{j=0}^{M-1}(s-\mathrm{max}\{0,s-d_{k,j}\}) \\
	\le & \gamma_{s}^{2}\frac{AL^{2}}{M^{3}}\sum\limits_{k=1}^{K}\sum\limits_{j=0}^{M-1}d_{k,j} = \gamma_{s}^{2}\frac{AL^{2}}{M^{2}}\sum\limits_{k=1}^{K}{\bar d}_{k}
	\end{align*}
	with the first inequality coming from Assumption 1. On the other hand, the expectation of $\tilde{Q}_2$ is bounded by
	
	\begin{align*}
	\mathbb{E}_{\bm{x}}\{\tilde{Q}_2 \}=&  -(\gamma_{s} - L\gamma_{s}^{2})\mathbb{E}_{\bm{x}}\{\sum\limits_{k=1}^{K}(\bm{\bar g}_{\bm{\theta}_{q(k)}}^{U_{s}})^{T}\Big(\bm{\mathfrak{g}}_{\bm{\theta}_{q(k)}}^{U_{s}^{'}} - \bm{\bar g}_{\bm{\theta}_{q(k)}}^{U_{s}}\Big)\} =-(\gamma_{s} - L\gamma_{s}^{2})\sum\limits_{k=1}^{K}(\bm{\bar g}_{\bm{\theta}_{q(k)}}^{U_{s}})^{T}\Big(\bm{\bar \mathfrak{g}}_{\bm{\theta}_{q(k)}}^{U_{s}^{'}} - \bm{\bar g}_{\bm{\theta}_{q(k)}}^{U_{s}}\Big)\\
	\le & \frac{\gamma_{s} - L\gamma_{s}^{2}}{2}\sum\limits_{k=1}^{K}\Big\lVert\bm{\bar g}_{\bm{\theta}_{q(k)}}^{U_{s}}\Big\lVert_{2}^{2} + \frac{\gamma_{s} - L\gamma_{s}^{2}}{2}\tilde{P}_1
	\end{align*}
	where the second equality follows by the unbiased gradient using SGD, and the inequality comes from $\pm \bm{x}^{T}\bm{y}\le \frac{1}{2}\lVert\bm{x}\lVert_2^2 + \frac{1}{2}\lVert\bm{y}\lVert_2^2$.
	
	Taking the expectation of both sides in \eqref{eq_appendix_3} and substituting $\tilde{Q}_1$ and $\tilde{Q}_2$, the inequality is rewritten as
	\begin{align}\nonumber
	\mathbb{E}_{\bm{x}}\{f(\bm{\theta}^{U_{s+1}})\} \le& f(\bm{\theta}^{U_{s}})- (\gamma_{s} - \frac{L\gamma_{s}^{2}}{2})\sum\limits_{k=1}^{K}\Big\lVert\bm{\bar g}_{\bm{\theta}_{q(k)}}^{U_{s}}\Big\lVert_2^2 + \frac{AL}{M}\gamma_{s}^{2} + L\gamma_{s}^{2}\tilde{P}_1+\frac{\gamma_{s} - L\gamma_{s}^{2}}{2}\sum\limits_{k=1}^{K}\Big\lVert\bm{\bar g}_{\bm{\theta}_{q(k)}}^{U_{s}}\Big\lVert_{2}^{2} + \frac{\gamma_{s} - L\gamma_{s}^{2}}{2}\tilde{P}_1\\\nonumber
	=&f(\bm{\theta}^{U_{s}})- \frac{\gamma_{s}}{2}\Big\lVert\bm{\bar g}_{\bm{\theta}}^{U_{s}}\Big\lVert_2^2 + \frac{\gamma_{s}+ L\gamma_{s}^{2}}{2}\tilde{P}_{1}+ \frac{AL}{M}\gamma_{s}^{2} \\\nonumber
	\le&f(\bm{\theta}^{U_{s}})- \frac{\gamma_{s}}{2}\Big\lVert\bm{\bar g}_{\bm{\theta}}^{U_{s}}\Big\lVert_2^2 + \frac{\gamma_{s}+ L\gamma_{s}^{2}}{2}\gamma_{s}^{2}\frac{AL^{2}}{M^{2}}\sum\limits_{k=1}^{K}{\bar d}_{k}+ \frac{AL}{M}\gamma_{s}^{2}\\\nonumber
	=&f(\bm{\theta}^{U_{s}}) - \frac{\gamma_{s}}{2}\Big\lVert\bm{\bar g}_{\bm{\theta}}^{U_{s}}\Big\lVert_2^2+ \gamma_{s}^{2}\Big( \frac{AL}{M} + \frac{\gamma_{s}+ L\gamma_{s}^{2}}{2}L\frac{AL}{M^{2}}\sum\limits_{k=1}^{K}{\bar d}_{k}\Big)\\\label{eq_ap_}
	\le&f(\bm{\theta}^{U_{s}}) - \frac{\gamma_{s}}{2}\Big\lVert\bm{\bar g}_{\bm{\theta}}^{U_{s}}\Big\lVert_2^2+ \gamma_{s}^{2}\frac{AL}{M}\Big( 1+ \frac{1}{M}\sum\limits_{k=1}^{K}{\bar d}_{k}\Big)
	\end{align}
	where the last inequality follows from $L\gamma_{s}\le 1$ such that $\frac{\gamma_{s}+ L\gamma_{s}^{2}}{2}L = \frac{1}{2}(L\gamma_{s} + (L\gamma_{s})^{2})\le 1$. The proof is now completed.
\end{proof}	

\section*{Supplementary material B: Proof of Theorem \ref{thm_avg_gra}}
\begin{proof}
	By moving $\frac{\gamma_{s}}{2}\Big\lVert\bm{\bar g}_{\bm{\theta}}^{U_{s}}\Big\lVert_2^2$ and $\mathbb{E}_{\bm{x}}\{f(\bm{\theta}^{U_{s+1}})\}$ to the LHS and the RHS of \eqref{eq_convergence} respectively, and multiplying both sides by $2$, we have
	\begin{align}\label{eq_thm_weighted_avg}
	\gamma_{s}\Big\lVert\bm{\bar g}_{\bm{\theta}}^{U_{s}}\Big\lVert_2^2 \le 2(f(\bm{\theta}^{U_{s}}) - \mathbb{E}_{\bm{x}}\{f(\bm{\theta}^{U_{s+1}})\}) + 2\gamma_{s}^{2}\frac{AL}{M}\Big( 1+ \frac{1}{M}\sum\limits_{k=1}^{K}{\bar d}_{k}\Big).
	\end{align} 
	Take full expectation on both sides of \eqref{eq_thm_weighted_avg}, and it leads to
	\begin{align}\label{eq_thm_weighted_avg2}
	\gamma_{s}\mathbb{E}\{\Big\lVert\bm{\bar g}_{\bm{\theta}}^{U_{s}}\Big\lVert_2^2\} \le 2(\mathbb{E}\{f(\bm{\theta}^{U_{s}})\} - \mathbb{E}\{f(\bm{\theta}^{U_{s+1}})\}) + 2\gamma_{s}^{2}\frac{AL}{M}\Big( 1+ \frac{1}{M}\sum\limits_{k=1}^{K}{\bar d}_{k}\Big).
	\end{align} 
	By summing both sides of \eqref{eq_thm_weighted_avg2} from $0$ to $S-1$, and dividing it by $\mathbb{T}_{S} = \sum_{s=0}^{S-1}\gamma_{s}$, it becomes
	\begin{align*}
	\frac{1}{\mathbb{T}_{S}}\sum_{s=0}^{S-1}\gamma_{s}\mathbb{E}\{\Big\lVert\bm{\bar g}_{\bm{\theta}}^{U_{s}}\Big\lVert_2^2\} &\le \frac{2(f(\bm{\theta}^{0}) - \mathbb{E}\{f(\bm{\theta}^{U_{S}})\})}{\mathbb{T}_{S}} + \frac{2\frac{AL}{M}\Big( 1+ \frac{1}{M}\textstyle\sum_{k=1}^{K}{\bar d}_{k}\Big)\sum_{s=0}^{S-1}\gamma_{s}^{2} }{\mathbb{T}_{S}}\\
	&\le \frac{2(f(\bm{\theta}^{0}) - f(\bm{\theta}^{*}))}{\mathbb{T}_{S}} + \frac{2\frac{AL}{M}\Big( 1+ \frac{1}{M}\textstyle\sum_{k=1}^{K}{\bar d}_{k}\Big)\sum_{s=0}^{S-1}\gamma_{s}^{2} }{\mathbb{T}_{S}}.
	\end{align*}
	where the last inequality comes from $ f(\bm{\theta}^{*}) \le \mathbb{E}\{f(\bm{\theta}^{U_{S}})\}$.
\end{proof}
\section*{Supplementary material C: Proof of Theorem \ref{thm_constant_lr}}
\begin{proof}
	We start the proof from \eqref{eq_thm_weighted_avg2} as the constant learning rate is a special case in Theorem \ref{thm_avg_gra}. By setting $\gamma_{s}=\gamma$, Eq. \eqref{eq_thm_weighted_avg2} is rewritten as
	\begin{align}\label{eq_fix_lr1}
	\mathbb{E}\{\Big\lVert\bm{\bar g}_{\bm{\theta}}^{U_{s}}\Big\lVert_2^2\}\le \frac{2(\mathbb{E}\{f(\bm{\theta}^{U_{s}})\} - \mathbb{E}\{f(\bm{\theta}^{U_{s+1}})\})}{\gamma} + 2\gamma \frac{AL}{M}\Big( 1+ \frac{1}{M}\sum\limits_{k=1}^{K}{\bar d}_{k}\Big).
	\end{align}
	Summing both sides of \eqref{eq_fix_lr1} from $s=0$ to $S-1$ and dividing them by $S$, it leads to
	\begin{align}\nonumber
	\frac{1}{S}\sum\limits_{s=0}^{S-1}\mathbb{E}\{\Big\lVert\bm{\bar g}_{\bm{\theta}}^{U_{s}}\Big\lVert_2^2\} &\le\frac{2(f(\bm{\theta}^{0}) - \mathbb{E}\{f(\bm{\theta}^{U_{S}})\})}{\gamma S} + 2\gamma \frac{AL}{M}\Big( 1+ \frac{1}{M}\sum\limits_{k=1}^{K}{\bar d}_{k}\Big)\\\label{eq_constant_avg_gra}
	&\le \frac{2(f(\bm{\theta}^{0}) - f(\bm{\theta}^{*}))}{\gamma S} + 2\gamma \frac{AL}{M}\Big( 1+ \frac{1}{M}\sum\limits_{k=1}^{K}{\bar d}_{k}\Big).
	\end{align}
	Substituting $\gamma =\epsilon\sqrt{M(f(\bm{\theta}^{0}) - f(\bm{\theta}^{*}))/\Big({SAL}(1+ (1/M){\textstyle\sum}_{k=1}^{K}{\bar d}_{k})}\Big)$ into \eqref{eq_constant_avg_gra}, the RHS becomes
	\begin{align*}
	\frac{2(f(\bm{\theta}^{0}) - f(\bm{\theta}^{*})) + 2\gamma^{2}SAL\Big( 1 + (1/M)\textstyle\sum_{k=1}^{K}{\bar d}_{k}\Big)/M}{\gamma S} &= \frac{(2 + 2\epsilon^{2})(f(\bm{\theta}^{0}) - f(\bm{\theta}^{*}))}{ S\epsilon\sqrt{M(f(\bm{\theta}^{0}) - f(\bm{\theta}^{*}))/\Big({SAL}(1+ (1/M){\textstyle\sum}_{k=1}^{K}{\bar d}_{k})}\Big)}\\
	&= \frac{(2 + 2\epsilon^{2})}{\epsilon}\sqrt{AL(f(\bm{\theta}^{0}) - f(\bm{\theta}^{*}))\Big( 1 + (1/M)\textstyle\sum_{k=1}^{K}{\bar d}_{k}\Big)/(MS)}.
	\end{align*}
	Since the LHS of \eqref{eq_constant_avg_gra} is the average of $\mathbb{E}\{\Big\lVert\bm{\bar g}_{\bm{\theta}}^{U_{s}}\Big\lVert_2^2\}$ for $s=0,1,\dots,S-1$, we have
	\begin{align*}
	\underset{t\in\{0,1,\dots,S-1\}}{\mathrm{min}}\mathbb{E}\{\Big\lVert\bm{\bar g}_{\bm{\theta}}^{U_{s}}\Big\lVert_2^2\} \le \frac{(2 + 2\epsilon^{2})}{\epsilon}\sqrt{AL(f(\bm{\theta}^{0}) - f(\bm{\theta}^{*}))\Big( 1 + (1/M)\textstyle\sum_{k=1}^{K}{\bar d}_{k}\Big)/(MS)}
	\end{align*}
	which completes the proof.
\end{proof}

\end{document}